\documentclass{article}
\usepackage{ijcai19}
\usepackage{times}
\usepackage{amssymb,amsmath,stmaryrd,graphics}
\renewcommand{\phi}{\varphi}
\newtheorem{definition}{Definition}
\newtheorem{theorem}{Theorem}
\newtheorem{lemma}{Lemma}
\newtheorem{claim}{Claim}

\newenvironment{proof}{\noindent{\sf Proof.}}{\hfill $\boxtimes\hspace{2mm}$\linebreak}
\newcommand{\qed}{\hfill $\boxtimes\hspace{1mm}$}

\usepackage{CJKutf8}

\renewcommand{\phi}{\varphi}
\renewcommand{\epsilon}{\varepsilon}

\newcommand{\N}{{\sf N}}
\newcommand{\B}{{\sf B}}
\newcommand{\cN}{{\sf \overline{N}}}
\renewcommand{\S}{{\sf S}}
\newenvironment{proof-of-claim}{\noindent{\sc Proof of Claim.}}{\hfill $\boxtimes\hspace{2mm}$\linebreak}

\title{Blameworthiness and Sacrifice}
\title{Moral Blameworthiness and Sacrifice}
\title{The Limits of Morality in Strategic Games}
\author{Agent Theories and Models; Action, Change and Causality}

\begin{document}

\begin{CJK}{UTF8}{gkai}


\author{
Rui Cao$^1$\and
Pavel Naumov$^2$
\affiliations
$^1$University of British Columbia, Canada\\
$^2$Claremont McKenna College, USA\\
\emails
rui.cao@alumni.ubc.ca,
pgn2@cornell.edu
}

\maketitle

\begin{abstract}
A coalition is blameable for an outcome if the coalition had a strategy to prevent it. It has been previously suggested that the cost of prevention, or the cost of sacrifice, can be used to measure the degree of blameworthiness. The paper adopts this approach and proposes a modal logical system for reasoning about the degree of blameworthiness. The main technical result is a completeness theorem for the proposed system.
\end{abstract}


\section{Introduction}

On October 13th, 2011, a two-year-old girl was run over by a van in the city of Foshan in Guangdong Province, China. 18 people passed by and ignored her before a stranger moved the girl to the side of the road and notified her mom~\cite{zzx11chinese,bristow11bbc}. The girl, nicknamed ``Little Yueyue'' by Chinese media, later died in hospital~\cite{bbc17bbc}. This tragic incident is not unique. On April 21, 2017, a woman was crossing a street in the city of Zhumadian in the Chinese Province of Henan. She was struck by a taxi and laid on the road until an SUV run over her, neither drivers nor pedestrian stopped to help her. She later died from injuries~\cite{xhr17chinese,j17nytimes}. Although many Chinese were outraged by these events and blamed the passers-by for not helping, others did not. Trying to understand the people who does not blame passers-by, the New York Times quotes a post on Weibo, a Chinese social network: ``If I helped her to get up and sent her to the hospital, doctors would ask you to pay the medical bill. Her relatives would come and beat you up indiscriminately.''~\cite{j17nytimes}. This explanation refers to an even earlier event. On November 20, 2006, in the city of Nanjing, Jiangsu Province, an old woman fell down on a bus stop. A young man helped her to get up, escorted her to a hospital, and stayed there until she was examined by a doctor. He was later sued by the woman and her relatives. The court eventually ordered him to pay 40 percent of the medical costs. The verdict said that ``according to common sense'' it is highly likely that the man is responsible for the woman's fall because otherwise he ``would have left soon after sending the woman to the hospital instead of staying there for the surgical check"~\cite{s11chinadaily}. New York Times cites Dali L. Yang, a political scientist at the University of Chicago, saying  “In the aftermath of the Nanjing case, many Chinese worry about the victims turning around to blame the helpers, and thus feel unable to offer direct help.''~\cite{j17nytimes}. 

The cited above post on Weibo essentially says that the people should not be blamed for not doing something that would require them to sacrifice a lot. A similar position is argued by Shelly Kagan in his book {\em The Limits of Morality}, where he suggests that sacrifice could be one of possible ways to measure the degree of moral blameworthiness: the higher sacrifice is required to prevent something the less a person (or a group of people) should be blamed for not doing it~\cite{kagan91}.
Nelkin argues that  the sacrifice can be used as a degree of blameworthiness and praiseworthiness~[\citeyear{nelkin16nous}]. \cite{hk18aaai} suggested to measure sacrifice by its {\em cost}. They also noted that there are other ways to define degree of blameworthiness; for example, through probability with which the harm could have been prevented. 

The dilemma of balancing costs and moral responsibilities is being faced not only by passers-by in China. Insurance companies need to choose between refusing to pay for very expensive drug and potentially saving life of a patient. Car engineers choose between better safety equipment and its higher cost. Governments have to balance public support of life-saving   medical research with its cost.

In this paper we propose a sound and complete logical system that describes  modality $\B^s_C\phi$ meaning ``coalition $C$ is blameable for $\phi$ with degree $s$'', where the degree is defined as the cost of the sacrifice the coalition would have made to prevent $\phi$. Our work heavily builds on \cite{alnr11jlc} logic for resource bounded coalitions and \cite{nt19aaai} logic of blameworthiness in strategic games. In turn, both of these papers rooted in widely studied~\cite{vw05ai,b07ijcai,sgvw06aamas,abvs10jal,avw09ai,b14sr,gjt13jaamas,ge18aamas} Marc Pauly's logic of coalition power~\cite{p02}. Pauly gave a complete axiomatization of modality $\S_C\phi$ that stands for ``coalition $C$ has a strategy to achieve $\phi$''.  \cite{alnr11jlc} did the same for modality $\S^r_C\phi$ meaning ``coalition $C$ has a strategy to achieve $\phi$ using resources $r$''. \cite{nt19aaai} proposed a complete axiomatization of modality $\B_C\phi$, ``coalition $C$ is blamable for outcome $\phi$''. They used a common approach of defining blamable as ``statement $\phi$ is true but coalition $C$ had a strategy to prevent it''. This approach is known as the principle of alternative possibilities~\cite{f69tjop,w17} or counterfactual~\cite{c15cop} definition of blameworthiness. It is also a part of Halpern-Pearl formal definition of causality as a relation between sets of variables~\cite{h16}. Counterfactuals could also be used to define regret and several other emotions~\cite{ls11ai}. Although the principle of alternative possibilities is the most common way to define blameworthiness, other approaches has been explored too.  \cite{x98jpl} introduced a complete axiomatization of a modal logical system for reasoning about responsibility defined as taking actions that guarantee a certain outcome. \cite{bht09jancl} extended Xu's work from individual responsibility to group responsibility.

In this paper we propose a logic of blameworthiness with sacrifice  that describes universal property of modality $\B^s_C$, meaning ``statement $\phi$ is true but coalition $C$ had a strategy to prevent it at cost no more than $s$''. Our main technical results are the soundness and the completeness theorems for this logical system.

This paper is organized as follows. In the next section we define the formal syntax and semantics of our logical system. In Section~\ref{axioms section} we list and discuss its axioms. In Section~\ref{soundness section} and Section~\ref{completeness section} we prove, respectively, the soundness and the completeness of our system. Section~\ref{conclusion section} concludes. 

\section{Syntax and Semantics}\label{syntax and semantics section}

In this paper we assume a fixed set $\mathcal{A}$ of agents and a fixed set of propositional variables $\sf Prop$. By a coalition we mean an arbitrary subset of set $\mathcal{A}$.

\begin{definition}\label{Phi}
$\Phi$ is the minimal set of formulae such that
\begin{enumerate}
    \item $p\in\Phi$ for each variable $p\in {\sf Prop}$,
    \item $\phi\to\psi,\neg\phi\in\Phi$ for all formulae $\phi,\psi\in\Phi$,
    \item $\N\phi$, $\B^s_C\phi\in\Phi$ for each coalition $C\subseteq\mathcal{A}$, each real number $s\ge 0$, and each formula $\phi\in\Phi$. 
\end{enumerate}
\end{definition}
In other words, language $\Phi$ is defined by  grammar:
$$
\phi := p\;|\;\neg\phi\;|\;\phi\to\phi\;|\;\N\phi\;|\;\B^s_C\phi.
$$
Informally, modality $\B_C^s\phi$ means ``coalition $C$ is blameable for statement $\phi$ with degree $s$''. The other modality in our system is $\N\phi$, which stands for ``statement $\phi$ is universally true in the given game''. By $\cN\phi$ we mean formula $\neg\N\neg\phi$. We assume that conjunction $\wedge$ and disjunction $\vee$ are defined in the standard way.

By $X^Y$ we mean the set of all functions from set $Y$ to $X$.
\begin{definition}\label{game definition}
A game is a tuple $\left(\Delta,\|\cdot\|,d_0,\Omega,P,\pi\right)$, where 
\begin{enumerate}
    \item $\Delta$ is a nonempty set of ``actions",
    \item $\Omega$ is a set of ``outcomes",
    \item $\|d\|$ is a non-negative real number for each $d\in \Delta$, called the cost of action $d$,
    \item $d_0\in \Delta$ is a zero-cost action: $\|d_0\|=0$,
    \item a set of ``plays" $P$ is an arbitrary set of pairs $(\delta,\omega)$ such that $\omega\in\Omega$ is an outcome and  $\delta\in\Delta^\mathcal{A}$ is a function from set $\mathcal{A}$ to set $\Delta$, called ``complete action profile'',
    \item $\pi$ is a function that maps $\sf Prop$ into subsets of $P$.
\end{enumerate}
\end{definition} 
For example, in the case of Little Yueyue from the introduction, each of the passers-by had two available actions: to help or to ignore. Thus, $\Delta=\{\mbox{help},\mbox{ignore}\}$. For the purpose of this example, we assume that any passer-by volunteering to help, would have to pay \yen 1000 towards Little Yueyue's medical bill: $\|\mbox{help}\|=1000$. We also assume that ignoring is a zero-cost option:  $\|\mbox{ignore}\|=0$. There are two possible outcomes: either Yueyue stays alive or dies. Thus, $\Omega=\{\mbox{alive},\mbox{dead}\}$.  The set of plays $P$ describes all possible combinations of actions and outcomes in the game. In our case, these combinations are listed as separate lines of the table in Figure~\ref{plays figure}. Although there have been 19 passers-by mentioned in Little Yueyue tragic story (18 of who decided to ignore and one who decided to help), to keep our example simple, we assume that there were only three agents: $a_1$, $a_2$, and $a_3$. If any of the first two of them decided to help, Little Yueyue would be alive. We assume that the third agent arrived too late to save her life, see Figure~\ref{plays figure}.

\begin{figure}[ht]
\begin{center}
\begin{tabular}{ l l l | c}
 $\delta(a_1)$ & $\delta(a_2)$ & $\delta(a_3)$ & $\omega$\\ \hline
ignore & ignore & ignore & dead \\ 
ignore & ignore & help & dead \\ 
ignore & help & ignore & alive \\ 
ignore & help & help & alive \\ 
help & ignore & ignore & alive \\ 
help & ignore & help & alive \\ 
help & help & ignore & alive \\ 
help & help & help & alive 
\end{tabular}
\end{center}
    \caption{Set of plays $P$.}
    \label{plays figure}
\end{figure}

Note that $\pi$ maps propositional variables not into set of outcomes, but into sets of plays. This is because we allow statements represented by atomic propositions to refer not only to outcome, but to actions as well. An example of such statement in our case is ``if agent $a_2$ helps, then Little Yueyue stays alive''. 

One can imagine a fixed ``tax'' added to costs of all actions in the game. Such uniform tax constitutes fixed overhead costs and should not be used to measure the sacrifice. To avoid this situation in Definition~\ref{sat} we assume existence of zero-cost action $d_0$. This assumption is significant because without it the Monotonicity axiom of our system would not be valid. A similar assumption is made in the logic for resource bounded coalitions  \cite{alnr11jlc}. 
 
\begin{definition}\label{cost of profile definition}
For any action profile $\gamma\in\Delta^\mathcal{C}$ of a coalition $C$ by $\|\gamma\|$ we mean the total cost of the action profile to the coalition: $\|\gamma\|=\sum_{a\in C}\|a\|$.
\end{definition}  
For any functions $f$ and $g$, we write $f=_Xg$, if $f(x)=g(x)$ for each $x\in X$.

\begin{definition}\label{sat} 
For any formula $\phi$ and any play $(\delta,\omega)\in P$ of a game $\left(\Delta,\|\cdot\|,d_0,\Omega,P,\pi\right)$, the satisfiability relation $(\delta,\omega)\Vdash\phi$ is defined recursively as follows:
\begin{enumerate}
    \item $(\delta,\omega)\Vdash p$ if $(\delta,\omega)\in \pi(p)$, where $p\in {\sf Prop}$,
    \item $(\delta,\omega)\Vdash \neg\phi$ if $(\delta,\omega)\nVdash \phi$,
    \item $(\delta,\omega)\Vdash\phi\to\psi$ if $(\delta,\omega)\nVdash\phi$ or $(\delta,\omega)\Vdash\psi$,
    \item $(\delta,\omega)\Vdash\N\phi$ if $(\delta',\omega')\Vdash\phi$ for each play $(\delta',\omega')\in P$,
    \item $(\delta,\omega)\Vdash\B^s_C\phi$ if $(\delta,\omega)\Vdash\phi$ and there is $\gamma\in \Delta^C$ such that $\|\gamma\|\le s$ and for each play $(\delta',\omega')\in P$, if $\gamma=_C\delta'$, then $(\delta',\omega')\nVdash\phi$.
\end{enumerate}
\end{definition}

Note that item 5 of Definition~\ref{sat} takes into account potential costs under action profile $\delta'=_C\gamma$ to coalition $C$ and ignores the actual costs under the action profile $\delta$. We refer to this way of defining the sacrifice as the {\em absolute sacrifice}. Alternatively, by the {\em relative sacrifice} we mean the difference between the costs under these two profiles or {\em how much more} it would cost the coalition to prevent undesired outcome comparing to the current costs. 

The use of the absolute sacrifice as a degree of blameworthiness makes sense in many situations. For example, let us assume that in the Little Yueyue example agent $a_1$ was heading to a medical supply store to buy a new \yen 1000 wheelchair for his ill child. The agent is now facing a moral choice between (a) zero-cost option of doing nothing, (b) spending  \yen 1000 on a new wheelchair, and (c) spending  \yen 1000 on a medical bill. If the agent were to choose to help instead of buying the wheelchair, his relative sacrifice would be zero. In this case, the absolute sacrifice of \yen 1000 is probably a better measure of the degree of blameworthiness. At the same time, relative sacrifice makes sense in situation like blame for the result of corner-cutting in safety, when a small additional expense could prevent a tragic incident. \cite{hk18aaai} use relative sacrifice as their measure of degree of blameworthiness.  

\section{Axioms}\label{axioms section}

In addition to the propositional tautologies in  language $\Phi$, our logical system contains the following axioms:

\begin{enumerate}
    \item Truth: $\N\phi\to\phi$ and $\B^s_C\phi\to\phi$,
    \item Distributivity: $\N(\phi\to\psi)\to(\N\phi\to \N\psi)$,
    \item Negative Introspection: $\neg\N\phi\to\N\neg\N\phi$,
    \item None to Blame: $\neg\B^s_\varnothing\phi$,
    \item Monotonicity: $\B^s_C\phi\to\B^t_D\phi$, where $C\subseteq D$ and $s\le t$,
    \item Joint Responsibility: if $C\cap D=\varnothing$, then\\ $\cN\B^s_C\phi\wedge\cN\B^t_D\psi\to (\phi\vee\psi\to\B^{s+t}_{C\cup D}(\phi\vee\psi))$,
    \item Blame for Cause:\\ $\N(\phi\to\psi)\to(\B^s_C\psi\to(\phi\to \B^s_C\phi))$,
    
    \item Fairness: $\B^s_C\phi\to\N(\phi\to\B^s_C\phi)$.
\end{enumerate}

These axioms are the same as the axioms of the logic of blameworthiness~\cite{nt19aaai} except for the sacrifice superscript being added. The Truth, the Distributivity, and the Negative Introspection axioms for modality $\N$ capture the fact that this is an S5-modality. The Truth axiom for modality $\B$ states that a coalition can only be blamed for something which is true. The None to Blame axiom says that the empty coalition can not be blamed for anything.

The Monotonicity axiom states that if a smaller coalition $C$ can be blamed for not preventing an outcome at cost at most $s$, then any larger coalition $D$ can also be blamed for not preventing the outcome at cost at most $t$, where $t\ge s$. This axiom is valid because each agent in set $D\setminus C$ could use the zero-cost action. One may question our underlying assumption that a larger coalition should be blamed for wrongdoings of its part. This assumption is consistent, for example, with how the entire millennial generation is blamed in the media for decline in sales of beer, paper napkins, and motorcycles~\cite{s18foxnews}.

The Joint Responsibility axiom shows how blames of two disjoint coalitions can be combined into a blame of their union. It resembles the Cooperation axiom for resource-bounded coalitions~\cite{alnr11jlc}: if $C\cap D=\varnothing$, then
$
\S^p_C(\phi\to\psi)\to(\S^q_D\phi\to \S^{p+q}_{C\cup D}\psi)
$.

To understand the Blame for Cause axiom note that formula $\N(\phi\to\psi)$ means that $\phi$ implies $\psi$ for each play of the game. In this case we say that $\phi$ is a {\em cause} of $\psi$. The axiom says that if a coalition is responsible for a statement, then it is also responsible for its cause as long as the cause is true.


The Fairness axiom states that if a coalition is blamed for $\phi$, then it should be blamed for $\phi$ each time when $\phi$ is true.

We write $\vdash\phi$ if formula $\phi$ is provable from the axioms of our system using the Modus Ponens and
the Necessitation inference rules:
$$
\dfrac{\phi,\phi\to\psi}{\psi},
\hspace{20mm}
\dfrac{\phi}{\N\phi}.
$$

The next lemma generalizes the Joint Responsibility axiom from two to multiple coalitions. Its proof is identical to the proof of the corresponding result in~\cite[Lemma 5]{nt19aaai} with the superscript added.
\begin{lemma}\label{super joint responsibility lemma}
For any integer $n\ge 0$,  
$$
\{\cN\B^{t_i}_{D_i}\chi_i\}_{i=1}^n,\chi_1\vee\dots\vee\chi_n
\vdash \B^{t_1+\dots+t_n}_{D_1\cup\dots\cup D_n}(\chi_1\vee \dots\vee\chi_n),
$$
where  sets $D_1,\dots,D_n$ are  pairwise disjoint.\qed
\end{lemma}

The following two lemmas capture well-know properties of S5 modality. Their proofs, for example, could be found in~\cite{nt18aamas}.
\begin{lemma}\label{super distributivity}
If $\phi_1,\dots,\phi_n\vdash\psi$, then $\N\phi_1,\dots,\N\phi_n\vdash\N\psi$. \qed
\end{lemma}

\begin{lemma}[Positive Introspection]\label{positive introspection lemma}
$\vdash \N\phi\to\N\N\phi$. \qed
\end{lemma}



We conclude this section with an example of a formal proof in our logical system. This example will be used later in the proof of the completeness.

\begin{lemma}\label{five plus plus}
For any integer $n\ge 0$ and any disjoint sets $D_1,\dots,D_n\subseteq C$ if $t_1+\dots+t_n\le s$, then
$$
\{\cN\B^{t_i}_{D_i}\chi_i\}_{i=1}^n,\N(\phi\to\chi_1\vee\dots\vee\chi_n)\vdash\N(\phi\to\B^{s}_C\phi).
$$
\end{lemma}
\begin{proof}
By Lemma~\ref{super joint responsibility lemma},
$$
\{\cN\B^{t_i}_{D_i}\chi_i\}_{i=1}^n,\chi_1\vee\dots\vee\chi_n\vdash \B^{t_1+\dots+t_n}_{D_1\cup\dots\cup D_n}(\chi_1\vee\dots\vee\chi_n).
$$
Hence, by the Monotonicity axiom, 
$$
\{\cN\B^{t_i}_{D_i}\chi_i\}_{i=1}^n,\chi_1\vee\dots\vee\chi_n\vdash \B^{s}_{C}(\chi_1\vee\dots\vee\chi_n).
$$
Then, by the Modus Ponens inference rule,
$$
\{\cN\B^{t_i}_{D_i}\chi_i\}_{i=1}^n,\phi,\phi\to\chi_1\vee\dots\vee\chi_n\vdash \B^{s}_C(\chi_1\vee\dots\vee\chi_n).
$$
Hence, by the Truth axiom and the Modus Ponens  rule,
$$
\{\cN\B^{t_i}_{D_i}\chi_i\}_{i=1}^n,\phi,\N(\phi\to\chi_1\vee\dots\vee\chi_n)\vdash \B^{s}_C(\chi_1\vee\dots\vee\chi_n).
$$
At the same time,
$$\N(\phi\to\chi_1\vee\dots\vee\chi_n)\to(\B^{s}_C(\chi_1\vee\dots\vee\chi_n)\to(\phi\to\B^{s}_C\phi))$$
is an instance of the Blame for Cause axiom. Thus, by the Modus Ponens inference rule applied twice,
$$
\{\cN\B^{t_i}_{D_i}\chi_i\}_{i=1}^n,\phi,\N(\phi\to\chi_1\vee\dots\vee\chi_n)\vdash\phi\to\B^s_C\phi.
$$
Then, by the Modus Ponens inference rule,
$$
\{\cN\B^{t_i}_{D_i}\chi_i\}_{i=1}^n,\phi, \N(\phi\to\chi_1\vee\dots\vee\chi_n)\vdash\B^s_C\phi.
$$
Hence, by the deduction lemma,
$$
\{\cN\B^{t_i}_{D_i}\chi_i\}_{i=1}^n,\N(\phi\to\chi_1\vee\dots\vee\chi_n)\vdash\phi\to\B^s_C\phi.
$$
Thus, by Lemma~\ref{super distributivity},
$$
\{\N\cN\B^{t_i}_{D_i}\chi_i\}_{i=1}^n,\N\N(\phi\to\chi_1\vee\dots\vee\chi_n)\vdash\N(\phi\to\B^s_C\phi).
$$
Then, by the definition of modality $\cN$, the Negative Introspection axiom, and the Modus Ponens inference rule,
$$
\{\cN\B^{t_i}_{D_i}\chi_i\}_{i=1}^n,\N\N(\phi\to\chi_1\vee\dots\vee\chi_n)\vdash\N(\phi\to\B^s_C\phi).
$$
Therefore, by Lemma~\ref{positive introspection lemma} and the Modus Ponens inference rule, the statement of the lemma is true. 
\end{proof}

\section{Soundness}\label{soundness section}

In this section we prove soundness of our logical system. The soundness of S5 axioms (the Truth, the Distributivity, and the Negative Introspection) for modality $\N$ is straightforward. Below we prove the soundness of the remaining axioms for any arbitrary play $(\delta,\omega)\in P$ of an arbitrary game $(\Delta,\|\cdot\|,d_0,\Omega,P,\pi)$. 

\begin{lemma}
$(\delta,\omega)\nVdash \B^s_\varnothing\phi$, for all $s \ge 0$. 
\end{lemma}
\begin{proof}
Suppose that $(\delta,\omega)\Vdash \B^s_\varnothing\phi$. Thus,
by Definition~\ref{sat}, $(\delta,\omega)\Vdash \phi$ and there is $\gamma\in\Delta^C$ such that $\|\gamma\|\le s$ and for each play $(\delta',\omega')$, if $\gamma=_\varnothing\delta'$, then $(\delta',\omega')\nVdash\phi$. Consider play $(\delta,\omega)$ and note that statement $\gamma=_\varnothing\delta$ is vacuously true. Thus, $(\delta,\omega)\nVdash\phi$, which is a contradiction.
\end{proof}

\begin{lemma}
For all sets $C,D\subseteq \mathcal{A}$ and all $s,t\ge 0$, if $C\subseteq D$, $s\leq t$, and $(\delta,\omega)\Vdash \B^s_C\phi$, then $(\delta,\omega)\Vdash \B^t_D\phi$.
\end{lemma}
\begin{proof}
By Definition~\ref{sat}, assumption $(\delta,\omega)\Vdash \B^s_C\phi$ implies that (i) $(\delta,\omega)\Vdash\phi$ and (ii) there is $\gamma\in \Delta^C$ such that $\|\gamma\|\le s$ and for each play $(\delta',\omega')\in P$, if $\gamma=_C\delta'$, then $(\delta',\omega')\nVdash\phi$. Define action profile $\gamma'\in\Delta^D$ as follows:
$$
\gamma'(a)=
\begin{cases}
\gamma(a), & \mbox{if } a \in C\\
d_0, & \mbox{otherwise}.
\end{cases}
$$
Because $d_0$ is a zero-cost action,  $\|\gamma'\|=\|\gamma\|=s\le t$ by Definition~\ref{cost of profile definition} and the assumption $s\le t$. Consider any play $(\delta',\omega')\in P$ such that $\gamma'=_D\delta'$. By Definition~\ref{sat} and because $(\delta,\omega)\Vdash\phi$, it suffices to show that $(\delta',\omega')\nVdash\phi$. Indeed, $\gamma=_C\gamma'=_C\delta'$. Therefore, $(\delta',\omega')\nVdash\phi$ by the choice of action profile $\gamma$.
\end{proof}

\begin{lemma}
For all $C,D\subseteq\mathcal{A}$ and all $s,t\ge 0$, if $C\cap D=\varnothing$, $(\delta,\omega)\Vdash \cN\B^s_C\phi$, $(\delta,\omega)\Vdash \cN\B^t_D\psi$, and $(\delta,\omega)\Vdash \phi\vee\psi$, then $(\delta,\omega)\Vdash \B^{s+t}_{C\cup D}(\phi\vee\psi)$.
\end{lemma}
\begin{proof}
By Definition~\ref{sat} and the definition of modality $\cN$, assumption  $(\delta,\omega)\Vdash \cN\B^s_C\phi$ implies that there is a play $(\delta',\omega')\in P$ such that $(\delta',\omega')\Vdash \B^s_C\phi$. Thus, by Definition~\ref{sat}, there is an action profile $\gamma_1\in \Delta^C$ such that $\|\gamma_1\|\le s$ and 
\begin{equation}\label{eq gamma 1}
    \forall \delta''\forall\omega''((\delta'',\omega'')\in P\wedge \gamma_1=_C\delta''\to(\delta'',\omega'')\nVdash \phi).
\end{equation}
Similarly, assumption $(\delta,\omega)\Vdash \cN\B^t_D\psi$ implies that there is an action profile $\gamma_2\in \Delta^D$ such that $\|\gamma_2\|\le t$ and 
\begin{equation}\label{eq gamma 2}
    \forall \delta''\forall\omega''((\delta'',\omega'')\in P\wedge\gamma_2=_C\delta''\to(\delta'',\omega'')\nVdash \psi).
\end{equation}
Consider action profile $\gamma\in\Delta^{C\cup D}$  such that
\begin{equation}\label{definition of gamma}
 \gamma(a)=
\begin{cases}
\gamma_1(a), & \mbox{if } a\in C,\\
\gamma_2(a), & \mbox{if } a\in D.
\end{cases}  
\end{equation}
Action profile $\gamma$ is well-defined because $C\cap D=\varnothing$ by the assumption of the lemma. Note that $\|\gamma\|=\|\gamma_1\|+\|\gamma_2\|$ by Definition~\ref{cost of profile definition} and inequalities $\|\gamma_1\|\le s$ and $\|\gamma_2\|\le t$.

Then, by Definition~\ref{sat} and the assumption $(\delta,\omega)\Vdash \phi\vee\psi$, it suffices to show that 
 for each play $(\delta''',\omega''')$, if $\gamma=_{C\cup D}\delta'''$, then $(\delta''',\omega''')\nVdash\phi\vee\psi$. Indeed, equality $\gamma=_{C\cup D}\delta'''$ implies  $\gamma=_{C}\delta'''$. Thus, $\gamma_1=_{C}\delta'''$ by equation~(\ref{definition of gamma}). Hence, $(\delta''',\omega''')\nVdash\phi$ by statement~(\ref{eq gamma 1}). Similarly, $(\delta''',\omega''')\nVdash\psi$ by statement~(\ref{eq gamma 2}). Then, $(\delta''',\omega''')\nVdash\phi\vee\psi$, by Definition~\ref{sat}.
\end{proof}

\begin{lemma}
If $(\delta,\omega)\Vdash \N(\phi\to\psi)$, $(\delta,\omega)\Vdash \B^s_C\psi$, and $(\delta,\omega)\Vdash \phi$, then $(\delta,\omega)\Vdash \B^s_C\phi$.
\end{lemma}
\begin{proof}
By Definition~\ref{sat}, assumption $(\delta,\omega)\Vdash \B^s_C\psi$ implies that there is an action profile $\gamma\in\Delta^C$ such that $\|\gamma\|\le s$ and for each play $(\delta',\omega')\in P$, if $\gamma=_C\delta'$, then $(\delta',\omega')\nVdash\psi$. 

Thus, by Definition~\ref{sat} and assumption $(\delta,\omega)\Vdash \N(\phi\to\psi)$, for each play $(\delta',\omega')$, if $\gamma=_C\delta'$, then $(\delta',\omega')\nVdash\phi$. Therefore, assumption $(\delta,\omega)\Vdash \phi$ implies that $(\delta,\omega)\Vdash \B^s_C\phi$ again by Definition~\ref{sat}.
\end{proof}

\begin{lemma}
If $(\delta,\omega)\Vdash \B^s_C\phi$, then $(\delta,\omega)\Vdash \N(\phi\to\B^s_C\phi)$.
\end{lemma}
\begin{proof}
Suppose that $(\delta,\omega)\nVdash \N(\phi\to\B^s_C\phi)$. Thus, by Definition~\ref{sat}, there is a play $(\delta',\omega')\in P$, such that $(\delta',\omega')\nVdash\phi\to\B^s_C\phi$. Then, 
$(\delta',\omega')\Vdash\phi$ and $(\delta',\omega')\nVdash\B^s_C\phi$ by Definition~\ref{sat}.
Hence, again by Definition~\ref{sat}, for each action profile $\gamma\in\Delta^C$ such that $\|\gamma\|\le s$, there is a play $(\delta'',\omega'')\in P$, such that $\gamma=_C\delta''$ and  $(\delta'',\omega'')\nVdash\phi$. Therefore, $(\delta,\omega)\Vdash \B^s_C\phi$ by Definition~\ref{sat}.
\end{proof}

\section{Completeness}\label{completeness section}

We start the proof of
the completeness by defining the canonical game $G(\omega_0)=\left(\Delta,\|\cdot\|,d_0,\Omega,P,\pi\right)$ for each maximal consistent set of formulae $\omega_0$. 

\begin{definition}\label{canonical Delta}
Set $\Delta$ consists of a zero-cost action $d_0$ and  all triples $(\phi,C,s)$ such that $\phi\in\Phi$ is a formula, $C$ is a nonempty coalition, and $s$ is a non-negative real number.
\end{definition}

Informally, we consider actions as ``votes'' of agents. Zero-cost action $d_0$ could be interpreted as abstaining from voting. Action $(\phi,C,s)$ by an agent $a$ means that agent $a$ is voting as a part of coalition $C$ to prevent $\phi$ at the total cost $s$ to the whole coalition. If agent $a$ votes $(\phi,C,s)$, then statement $\phi$ is not necessarily false in the outcome. The vote aggregation mechanism is given in Definition~\ref{canonical play}.
Definition~\ref{canonical Delta} is substantially different from a similar definition in~\cite{nt19aaai}, where each action consists of just a single formula $\phi$.

\begin{definition}\label{canonical cost definition}
For each action $d\in \Delta$, let $\|d\|=0$ if $d=d_0$ and $\|d\|=\dfrac{s}{|C|}$ if $d=(\phi,C,s)$.
\end{definition}

Informally, $\|d\|=\frac{s}{|C|}$ means that the cost of each joint action is divided evenly between all members of the coalition. Note that size $|C|$ of coalition $C$ is non-zero by Definition~\ref{canonical Delta}.

\begin{definition}\label{canonical outcome}
The set of outcomes $\Omega$ is the set of all maximal consistent sets of formulae $\omega$ such that for each formula $\phi$ if $\N\phi\in \omega_0$, then $\phi\in \omega$.
\end{definition}

\begin{definition}\label{canonical play}
The set $P\subseteq \Delta^\mathcal{A}\times \Omega$ consists of all pairs $(\delta,\omega)$ such that for  any formula $\cN\B^s_C\psi\in \omega_0$, if $\delta(a)=(\psi,C,s)$ for each agent $a\in C$, then $\neg\psi\in \omega$.
\end{definition}
In other words, for each formula $\cN\B^s_C\psi\in \omega_0$, if each member of coalition $C$ votes as a part of $C$ to prevent $\phi$ at cost $s$, then $\phi$ is guaranteed to be false in the outcome.

\begin{definition}\label{canonical pi}
$\pi(p)=\{(\delta,\omega)\in P\;|\; p\in \omega\}$.
\end{definition}

As usual, the key part of the proof of the completeness is the induction, or ``truth'', lemma. In our case this is Lemma~\ref{induction lemma}. The next three lemmas are auxiliary lemmas used in the proof of Lemma~\ref{induction lemma}. 

\begin{lemma}\label{B child exists lemma}
For any play $(\delta,\omega)\in P$, any profile $\gamma\in\Delta^C$, and any formula $\neg(\phi\to \B^s_C\phi)\in \omega$, if $\|\gamma\|\le s$, then there is a play $(\delta',\omega')\in P$ such that $\gamma =_C\delta'$ and $\phi\in \omega'$.
\end{lemma}
\begin{proof}
Consider the following set $X$ of formulae:
\begin{eqnarray*}
&&\hspace{-5mm}\{\phi\}\;\cup\;\{\psi\;|\;\N\psi\in \omega_0\}\\
&&\hspace{-5mm}\cup\;\{\neg\chi\;|\;\cN\B^t_D\chi\in \omega_0, D\subseteq C,\forall a\in D(\gamma(a)=(\chi,D,t))\}.
\end{eqnarray*}
\begin{claim}
Set $X$ is consistent.
\end{claim}
\begin{proof-of-claim}
Suppose the opposite. Thus, there are formulae
\begin{eqnarray}
&&\N\psi_1,\dots,\N\psi_m\in \omega_0,\label{choice of psi-s}\\
&&\cN\B^{t_1}_{D_1}\chi_1,\dots,\cN\B^{t_n}_{D_n}\chi_n\in \omega_0,\label{choice of chi-s}
\end{eqnarray}
such that
\begin{eqnarray}
&&D_1,\dots,D_n\subseteq C,\label{choice of Ds}\\
&&\gamma(a)=(\chi_i,D_i,t_i)\mbox{ for all } a\in D_i, i\le n\label{choice of votes},\\
&&\psi_1,\dots,\psi_m,\neg\chi_1,\dots,\neg\chi_n\vdash\neg\phi.\label{choice of cons}
\end{eqnarray}
Without loss of generality, we can assume that formulae $\chi_1,\dots,\chi_n$ are distinct. Thus, assumption~(\ref{choice of votes}) implies that sets $D_1,\dots,D_n$ are pairwise disjoint. Hence, by Definition~\ref{canonical cost definition} and formula~(\ref{choice of votes}),
\begin{eqnarray*}
\|\gamma\|&=&\sum_{a\in C}\|\gamma(a)\|\ge \sum_{a\in D_1}\|\gamma(a)\|+\dots + \sum_{a\in D_n}\|\gamma(a)\|\\
&=&\sum_{a\in D_1}\|(\chi_1,D_1,t_1)\|+\dots + \sum_{a\in D_n}\|(\chi_n,D_n,t_n)\|\\
&=& \sum_{a\in D_1}\dfrac{t_1}{|D_1|} + \dots + \sum_{a\in D_n}\dfrac{t_n}{|D_n|}=t_1+\dots+t_n.
\end{eqnarray*}
Thus, by to the assumption $\|\gamma\|\le s$ of the lemma,
\begin{equation}\label{ts le s}
    t_1+\dots+t_n\le s.
\end{equation}

At the same tine, assumption~(\ref{choice of cons}) by the laws of propositional reasoning implies that
$$
\psi_1,\dots,\psi_m\vdash\phi\to\chi_1\vee\dots\vee\chi_n.
$$
Thus,
$
\N\psi_1,\dots,\N\psi_m\vdash\N(\phi\to\chi_1\vee\dots\vee\chi_n).
$
by Lemma~\ref{super distributivity},
Hence, 
$
    \omega_0\vdash\N(\phi\to\chi_1\vee\dots\vee\chi_n
$
by assumption~(\ref{choice of psi-s}).
Thus, by  Lemma~\ref{five plus plus}, using assumption~(\ref{choice of chi-s}), statement (\ref{ts le s}), and the fact that sets $D_1,\dots,D_n\subseteq C$ are pairwise disjoint,
$$
\omega_0\vdash \N(\phi\to\B^{s}_C\phi).
$$
Hence, $\N(\phi\to\B^s_C\phi)\in \omega_0$ because set $\omega_0$ is maximal. Then, $\phi\to\B^{s}_C\phi\in \omega$ by Definition~\ref{canonical outcome}. Thus, $\neg(\phi\to\B^{s}_C\phi)\notin \omega$ because set $\omega$ is consistent, which contradicts assumption $\neg(\phi\to\B^{s}_C\phi)\in \omega$ of the lemma.
Therefore, set $X$ is consistent.
\end{proof-of-claim}

Let $\omega'$ be any maximal consistent extension of set $X$. Thus, $\phi\in X\subseteq\omega'$ by the choice of sets $X$ and $\omega'$. Also, $\omega'\in\Omega$ by Definition~\ref{canonical outcome} and the choice of sets $X$ and $\omega'$.

Let the complete action profile $\delta'$ be defined as follows:
\begin{equation}\label{choice of delta'}
    \delta'(a)=
    \begin{cases}
    s(a), & \mbox{ if } a\in C,\\
    d_0, & \mbox{ otherwise}.
    \end{cases}
\end{equation}
Then, $s=_C\delta'$.

\begin{claim}
$(\delta',\omega')\in P$.
\end{claim}
\begin{proof-of-claim}
Consider any formula $\cN\B^t_D\chi\in \omega_0$ such that $\delta'(a)=(\chi,D,t)$ for each $a\in D$. By Definition~\ref{canonical play}, it suffices to show that $\neg\chi\in \omega'$. 

\noindent{\bf Case I:} $D\subseteq C$. Thus, $\neg\chi\in X$ by the definition of set $X$. Therefore, $\neg\chi\in \omega'$ by the choice of set $\omega'$.

\noindent{\bf Case II:} $D\nsubseteq C$. Consider any $a\in D\setminus C$. Thus, $\delta'(a)=d_0$ by equation~(\ref{choice of delta'}). At the same time, $\delta'(a)=(\chi,D,t)$ because $a\in D$. Therefore, $d_0=(\chi,D,t)$, which is a contradiction.
\end{proof-of-claim}
This concludes the proof of the lemma.
\end{proof}

\begin{lemma}\label{delta exists lemma}
For each outcome $\omega\in\Omega$, there is a complete action profile $\delta\in \Delta^\mathcal{A}$ such that $(\delta,\omega)\in P$.
\end{lemma}
\begin{proof}
Consider a complete action profile $\delta$ where $\delta(a)=d_0$ for all $a\in \mathcal{A}$. To show $(\delta,\omega)\in P$, consider any such formula $\cN\B^t_D\chi\in \omega_0$  that $\delta(a)=(\chi,D,t)$ for all $a\in D$. Due to Definition~\ref{canonical play}, it enough to prove that $\neg(\chi,D,t)\in \omega$.

\noindent{\bf Case I}: $D=\varnothing$. Hence, $\vdash\neg\B^t_D\chi$ by the None to Blame axiom. Then, $\vdash\N\neg\B^t_D\chi$ by the Necessitation inference rule. Thus, $\neg\N\neg\B^t_D\chi\notin\omega_0$ by the consistency of the set $\omega_0$. Therefore, $\cN\B^t_D\chi\notin\omega_0$ due to the definition of the modality $\cN$, which contradicts to the assumption $\cN\B^t_D\chi\in \omega_0$. 

\noindent{\bf Case II}: $D\neq\varnothing$. Hence, set $D$ contains at least one agent $a$. Then, $(\chi,D,t)=\delta(a)=d_0$ by the definition of profile $\delta$. Thus, $d_0=(\chi,D,t)$, which is a contradiction. 
\end{proof}

\begin{lemma}\label{N child exists lemma}
For each play $(\delta,\omega)\in P$ and each formula $\neg\N\phi\in \omega$,  there is a play $(\delta',\omega')\in P$ such that $\neg\phi\in \omega'$.
\end{lemma}
\begin{proof}
Let $X$ be the set $\{\neg\phi\}\;\cup\;\{\psi\;|\;\N\psi\in \omega_0\}$. Next, we prove the consistency of set $X$. Assume the opposite. Hence, there are formulae  $\N\psi_1,\dots,\N\psi_n\in \omega_0$
where
$
\psi_1,\dots,\psi_n\vdash\phi.
$
Thus, 
$
\N\psi_1,\dots,\N\psi_n\vdash\N\phi
$
due to Lemma~\ref{super distributivity}.
Then, $\omega_0\vdash\N\phi$ because $\N\psi_1,\dots,\N\psi_n\in \omega_0$. Thus,
$\omega_0\vdash\N\N\phi$ by Lemma~\ref{positive introspection lemma}.
Hence, it follows from assumption $\omega\in\Omega$ and Definition~\ref{canonical outcome} that $\N\phi\in\omega$. Thus, by the consistency of set $\omega$ $\neg\N\phi\notin\omega$, which contradicts the assumption of the lemma. Therefore, set $X$ is consistent.

Consider any maximal consistent extension $\omega'$ of set $X$. Observe that $\neg\phi\in X\subseteq\omega'$ due to the definition of set $X$. Finally, by Lemma~\ref{delta exists lemma}, there is a profile $\delta'$ where $(\delta',\omega')\in P$.
\end{proof}

\begin{lemma}\label{induction lemma}
$(\delta,\omega)\Vdash\phi$ iff $\phi\in\omega$ for any play $(\delta,\omega)\in P$ and any formula $\phi\in\Phi$.
\end{lemma}
\begin{proof}
The lemma will be shown by induction on structural complexity of formula $\phi$. If $\phi$ is a propositional variable, then the statement of the lemma follows from Definition~\ref{sat} and Definition~\ref{canonical pi}. The cases when $\phi$ is a negation or an implication, as usual, can be proven from the maximality and the consistency of set $\omega$.

Let formula $\phi$ have the form $\N\psi$.

\noindent $(\Rightarrow):$ Suppose $\N\psi\notin\omega$. Then, $\neg\N\psi\in\omega$ by the maximality of set $\omega$. Thus, there is a play $(\delta',\omega')\in P$ such that $\neg\psi\in \omega'$, by Lemma~\ref{N child exists lemma}. Hence, $\psi\notin \omega'$ because set $\omega'$ is consistent. Then, by the induction hypothesis, $(\delta',\omega')\nVdash\psi$. Therefore, $(\delta,\omega)\nVdash\N\psi$ by Definition~\ref{sat}.

\vspace{1mm}

\noindent $(\Leftarrow):$ Suppose $\N\psi\in\omega$. Then, $\neg\N\psi\notin\omega$ because set $\omega$ is consistent. Thus, $\N\neg\N\psi\notin\omega_0$ by Definition~\ref{canonical outcome}. Hence, $\omega_0\nvdash \N\neg\N\psi$ due to the maximality of set $\omega_0$. Then, by the Negative Introspection axiom, $\omega_0\nvdash\neg\N\psi$. Thus, $\N\psi\in \omega_0$ by the maximality of set $\omega_0$. Hence, $\psi\in \omega'$ for each outcome $\omega'\in\Omega$ by Definition~\ref{canonical outcome}. Hence, by the induction hypothesis, $(\delta',\omega')\Vdash\psi$ for all plays $(\delta',\omega')\in P$. Then, $(\delta,\omega)\Vdash\N\psi$ by Definition~\ref{sat}.

Let formula $\phi$ have the form $\B^s_C\psi$. 

\noindent $(\Rightarrow):$ Suppose 
that $\B^s_C\psi\notin \omega$. 

\noindent{\bf Case I:} $\psi\notin \omega$. Thus, $(\delta,\omega)\nVdash\psi$ by the induction hypothesis. Therefore, $(\delta,\omega)\nVdash\B^s_C\psi$ by Definition~\ref{sat}. 

\noindent{\bf Case II:}
$\psi\in \omega$. First we prove that $\psi\to\B^s_C\psi\notin \omega$. Suppose $\psi\to\B^s_C\psi\in \omega$. Thus, $\omega\vdash \B^s_C\psi$ by the Modus Ponens inference rule. Hence, by the  maximality of set $\omega$, we have $\B^s_C\psi\in \omega$, which contradicts the assumption $\B^s_C\psi\notin \omega$.

Since $\omega$ is a maximal set, statement $\psi\to\B^s_C\psi\notin \omega$ implies that $\neg(\psi\to\B^s_C\psi)\in \omega$. Hence, by Lemma~\ref{B child exists lemma}, for any action profile $\gamma\in \Delta^C$, if $\|\gamma\|\le s$, then there is a play $(\delta',\omega')$ where $\gamma=_C\delta'$ and $\psi\in \omega'$. Thus, by the induction hypothesis, for each profile $\gamma\in \Delta^C$, if $\|\gamma\|\le s$, then there is a play $(\delta',\omega')\in P$ such that $\gamma=_C\delta'$ and $(\delta',\omega')\Vdash \psi$. Therefore, $(\delta,\omega)\nVdash\B^s_C\psi$ by Definition~\ref{sat}.

\vspace{1mm}
\noindent $(\Leftarrow):$ Suppose that $\B^s_C\psi\in \omega$. Hence, $\omega\vdash\psi$ by the Truth axiom. Then, $\psi\in\omega$ by the maximality of the set $\omega$. Thus, $(\delta,\omega)\Vdash\psi$ by the induction hypothesis.

Define $\gamma\in \Delta^C$ to be an action profile such that $\gamma(a)=(\psi,C,s)$ for each agent $a\in C$.
\begin{claim}\label{gamma le s claim}
$\|\gamma\|\le s$.
\end{claim}
\begin{proof-of-claim}
If set $C$ is not empty, then, by Definition~\ref{cost of profile definition} and Definition~\ref{canonical cost definition},
$$
\|\gamma\|=\sum_{a\in C}\|(\psi,C,s)\|=
\sum_{a\in C}\dfrac{s}{|C|}=s.
$$

If set $C$ is empty, then $\|\gamma\|=0$ by Definition~\ref{cost of profile definition}. At the same time, $s\ge 0$ by Definition~\ref{Phi}. Therefore, $\|\gamma\|\le s$.
\end{proof-of-claim}

Consider any play $(\delta',\omega')\in P$ such that $\gamma=_C\delta'$. By Definition~\ref{sat} and Claim~\ref{gamma le s claim}, it suffices to show that  $(\delta',\omega')\nVdash \psi$.

Statement $\B^s_C\psi\in \omega$ implies that $\neg\B^s_C\psi\notin \omega$ because set $\omega$ is consistent. Thus, $\N\neg\B^s_C\psi\notin \omega_0$ by Definition~\ref{canonical outcome} and because $\omega\in\Omega$. Hence, $\neg\N\neg\B^s_C\psi\in \omega_0$ due to the maximality of the set $\omega_0$. Thus, $\cN\B^s_C\psi\in \omega_0$ by the definition of modality $\cN$.
Also, $\delta'(a)=\gamma(a)=(\psi,C,s)$ for each $a\in C$. Hence, $\neg\psi\in\omega'$ by Definition~\ref{canonical play} and the assumption $(\delta',\omega')\in P$. Then, $\psi\notin\omega'$ by the consistency of set $\omega'$. Therefore, $(\delta',\omega')\nVdash \psi$ by the induction hypothesis.
\end{proof}

Finally, we are prepared to state and prove the strong completeness of our logical system.
\begin{theorem}\label{completeness theorem}
If $X\nvdash\phi$, then there is a game and a play $(\delta,\omega)$ of the game where $(\delta,\omega)\Vdash\chi$ for all $\chi\in X$ and $(\delta,\omega)\nVdash\phi$.
\end{theorem}
\begin{proof}
Assume that $X\nvdash\phi$. Hence, set $X\cup\{\neg\phi\}$ is consistent. Choose $\omega_0$ to be any maximal consistent extension of set $X\cup\{\neg\phi\}$ and $G(\omega_0)=(\Delta,\|\cdot\|,d_0,\Omega,P,\pi)$ to be the canonical game defined above. Then, $\omega_0\in \Omega$ by Definition~\ref{canonical outcome} and the Truth axiom. 

By Lemma~\ref{delta exists lemma}, there exists an action profile $\delta\in \Delta^\mathcal{A}$ such that $(\delta,\omega_0)\in P$. Hence, $(\delta,\omega_0)\Vdash\chi$ for all $\chi\in X$ and $(\delta,\omega_0)\Vdash\neg\phi$ by Lemma~\ref{induction lemma} and the choice of set $\omega_0$. Therefore,  $(\delta,\omega_0)\nVdash\phi$ by Definition~\ref{sat}.
\end{proof}

\section{Conclusion}\label{conclusion section}

In this paper we combine the ideas from the logics of resource bounded coalitions~\cite{alnr11jlc} and blameworthiness~\cite{nt19aaai} into a logical system that captures the properties of a degree of blameworthiness. Following~\cite{hk18aaai}, the degree of blameworthiness is defined as the cost of sacrifice. The main technical result is the completeness theorem for our system.
\label{end of paper}

\bibliographystyle{named}
\bibliography{sp}

\end{CJK}
\end{document}